\documentclass{IEEEtran}

\usepackage{amsthm}
\theoremstyle{plain}
\usepackage[compress]{cite}
\usepackage{graphics}
\usepackage{epstopdf}
\usepackage{epsfig}
\usepackage{lineno}
\usepackage{array}
\usepackage{amsmath,amssymb}
\usepackage{mathrsfs}
\usepackage{multirow}
\usepackage{bm}
\usepackage{eqparbox}
\usepackage{algorithm}  
\usepackage{algorithmic}
\usepackage{subcaption,graphicx}
\newtheorem{myTheorem}{Theorem}
\newtheorem{myLemma}{Lemma}

\newtheorem{myCorollary}{Corollary}

\usepackage{amsfonts}
\usepackage{enumerate}
\usepackage{indentfirst}
\usepackage{ragged2e}
\usepackage{url}
\usepackage{wrapfig}
\usepackage{bbm}
\usepackage{soul}
\usepackage{pifont}

\allowdisplaybreaks[4]

\usepackage[utf8]{inputenc} % allow utf-8 input
\usepackage[T1]{fontenc}    % use 8-bit T1 fonts
\usepackage{nicefrac}       % compact symbols for 1/2, etc.
\usepackage[compress]{cite}

\usepackage{hyperref}

\begin{document}

\title{Discretize Relaxed Solution of Spectral Clustering via A Non-Heuristic Algorithm}

\author{Hongyuan Zhang and Xuelong Li,\IEEEmembership{~Fellow,~IEEE}\\

\thanks{This work is supported by The National Natural Science Foundation of China (No. 61871470).}

\thanks{The authors are with the School of Artificial Intelligence, OPtics and ElectroNics (iOPEN), Northwestern Polytechnical University, Xi'an 710072, Shaanxi, P. R. China.}

% \thanks{This work is supported by The National Natural Science Foundation of China (No. 61871470).}

\thanks{E-mail: hyzhang98@gmail.com, li@nwpu.edu.cn}

\thanks{Corresponding author: Xuelong Li}

\thanks{The source code is available at \url{https://github.com/hyzhang98/first-order-discretization}.}

}

\markboth{IEEE TRANSACTIONS ON NEURAL NETWORKS AND LEARNING SYSTEMS}
{Zhang \MakeLowercase{\textit{et al.}}: Discretize Relaxed Solution of Spectral Clustering via A Non-Heuristic Algorithm}

\maketitle

\begin{abstract}
    Spectral clustering and its extensions usually consist of two steps: 
    (1) constructing a graph and computing the relaxed solution; 
    (2) discretizing relaxed solutions. 
    Although the former has been extensively investigated, the discretization 
    techniques are mainly heuristic methods, \textit{e.g.}, $k$-means, 
    spectral rotation. 
    Unfortunately, the goal of the existing methods is not to find 
    a discrete solution that minimizes the original objective. 
    In other words, the primary drawback is the neglect of the original objective  
    when computing the discrete solution. 
    Inspired by the first-order optimization algorithms, we propose to  
    develop a first-order term to bridge the original problem and discretization algorithm, 
    which is the first non-heuristic to the best of our knowledge. 
    Since the non-heuristic method is aware of the original graph cut problem, 
    the final discrete solution is more reliable and achieves the preferable loss value. 
    We also theoretically show that the continuous optimum is beneficial to discretization algorithms 
    though simply finding its closest discrete solution is an existing heuristic algorithm which is also unreliable. 
    Sufficient experiments significantly show the superiority of our method. 
\end{abstract}


\section{Introduction}
Spectral clustering \cite{RCut,NCut,SC} has been widely applied in practice due to its ability 
to exploit the non-Euclidean property of data. 
Spectral clustering originates from the graph cut problem, 
\textit{e.g.}, Ratio Cut \cite{RCut}, Normalized Cut \cite{NCut}, Balanced Cut \cite{BalancedCut}, 
Improved Normalized Cut \cite{ImprovedNCut}. 
The procedure of spectral clustering and its variants usually consist of two phases: 
(1) Construct a graph and calculate the relaxed solution; 
(2) Compute the discrete solution. 
In general, 
the first step is to convert an arbitrary dataset into a graph so that the 
clustering is equivalent to partitioning a graph into several cohesive 
disjointed subsets of vertices, 
which is a well-known graph cut problem. 
Since most graph cut problems are NP-hard, most spectral clustering models 
turn to solve the continuously relaxed problem, 
which is usually easy to compute the optimum. 
After obtaining the continuous solution, an essential step is to compute 
an approximated discrete solution according to the continuous solution, 
which corresponds to step 2.

Compared with step 2, the strong extensions of spectral clustering \cite{CAN,AdaGAE,AnchorGAE,ProjectedClustering,AddedRef-1,AddedRef-2,AddedRef-3,AddedRef-4} 
prefer to focus on step 1, \textit{i.e.}, 
how to construct an effective graph that captures the potential topology of data. 
Specially, CLR \cite{CLR} and CAN \cite{CAN} attempt to directly construct a graph with 
$c$ connected components (where $c$ is the number of clusters) 
so that step 2 could be omitted. 

Although step 1 has been extensively investigated in recent decades, 
the study of step 2 is relatively limited. 
The most popular technique to discretize the continuous solution  
is to run $k$-means on the relaxed solutions \cite{NCut}. 
It is a heuristic method since it does not aim to find the optimal discrete optimum, 
even when $k$-means converges to its optimum. 
Literature \cite{SC} provides a convincing explanation: $k$-means can compute the nearly optimal 
partitions from the relaxed solution provided that the graph is easy to 
be cut to $c$ connected components. 
Another technique is to directly find the closest discrete solution regarding 
Euclidean distance, namely spectral rotation \cite{SR,ISR}. 
It can be also regarded as heuristic since the closest solution regarding 
Euclidean distance is usually not the optimal solution, 
which is elaborated in succeeding sections. 

Aiming at designing a reliable method to compute the discrete solution 
from the continuous optimum, we propose a non-heuristic discretization algorithm 
% to discretize the continuous optimum. 
and the contributions are summarized as follows:
\textbf{(1)} Inspired by the first-order gradient-based algorithms, a \textbf{non-heuristic} 
algorithm is proposed in this paper, which is the \textbf{first} non-heuristic method to the best of our knowledge. 
The proposed framework bridges the original graph cut functions and discretization 
algorithm via the \textbf{gradient}. 
% The optimization of the designed first-order term is much easier than the original cut problem. 
\textbf{(2)} Although simply finding the nearest discrete solution under Euclidean distance 
is unreliable, we theoretically show that starting from the continuous optimum is beneficial and meaningful. 
\textbf{(3)} Experiments strongly verify the effectiveness of our idea. 
The proposed method significantly outperforms other discretization methods on numerous datasets. 
% \begin{itemize}
%     \item Inspired by the first-order gradient-based algorithms, a non-heuristic 
%     algorithm is first proposed in this paper. 
%     The proposed framework bridges the original graph cut functions and discretization 
%     algorithm via the gradient. 
%     The designed first-order term is much easier to optimize, compared with the original cut problem. 
%     \item Although simply finding the nearest discrete solution under Euclidean distance 
%     is unreliable, we theoretically show that starting from the continuous optimum is beneficial and meaningful. 
%     \item Experiments strongly verify the effectiveness of our idea. 
%     The proposed method significantly outperforms other discretization methods on numerous datasets. 
% \end{itemize}

\section{Preliminary}

\subsection{Notations}
In this paper, all vectors and matrices are denoted by lower-case and 
upper-case letters in bold, respectively. 
Define $\mathcal{B}_{a \times b} = \{\bm Y \in \mathbb{R}^{a \times b} | Y_{ij} \in \{0, 1\}, \sum_{j=1}^b Y_{ij} = 1\}$. 
$[\bm U, \bm \Sigma, \bm V] \leftarrow \textrm{SVD}(\bm M)$ represents the singular-value decomposition procedure, 
where $\bm U$ is the left-singular vector, $\bm \Sigma$ is the singular value matrix, and $\bm V$ is the right-singular vector. 
$\mathbbm{1}\{\cdot\}$ is the indicator function. 
$\|\cdot\|$ represents the Frobenius-norm and $\ell_2$-norm for matrices and vectors, 
respectively. $\langle \cdot, \cdot \rangle$ denotes the inner-product. 
$n$ and $c$ denote the number of data points and clusters, respectively. 

\subsection{Revisit the Discretization Works}
As shown in \cite{RCut,NCut,SC,ImprovedNCut}, the spectral clustering with different graph cut problems
can be generally summarized as 
\begin{equation} \label{problem_raw}
    \min _{\bm G \in \mathcal{G}_{n \times c}, \bm Y \in \mathcal{B}_{n \times c}} {\rm tr} (\bm G^T \bm L \bm G) , 
\end{equation}
where $\bm L$ represents some Laplacian matrix, 
$\mathcal{G}_{n \times c} = \{f(\bm Y) | \bm Y \in \mathcal{B}_{n \times c}\}$, 
and $f(\bm Y)$ is some transformation of $\bm Y$. 
The row of $\bm G$ is usually 1-sparse so that the clustering assignments are 
directly given. 
Remark that the specific formulations of both $\bm L$ and $f(\bm Y)$ are 
decided by the used graph cut problem. 
The optimization of the above problem is NP-hard and most existing works 
turn to solve the continuously relaxed problem 
\begin{equation} \label{problem_raw_continuous}
    \min _{\bm F \in \mathcal{D}_f} {\rm tr}(\bm F^T \bm L \bm F) = \mathcal{L} (\bm F), 
\end{equation}
where $\mathcal{D}_f$ represents some continuous superset of the original 
feasible domain of $\bm G$, \textit{i.e.}, 
$\mathcal{G}_{n \times c} \subseteq \mathcal{D}_f$. 
After relaxation, the optimum could be calculated within polynomial 
time and the optimal solution is represented as $\bm F_*$. 
Each row vector $\bm f_*^i$ is regarded as the relaxed cluster indicator and 
the popular method \cite{NCut,SC} is to run $k$-means on $\{\bm f_*^i\}_{i=1}^n$. 
% which can be formulated as 
% \begin{equation}
% % $
% \bm Y = f^{-1}(\bm F_*) = k\textrm{-means}(\bm F_*) . 
% % $. 
% \end{equation}
Literature \cite{SC} provides a reasonable explanation for $k$-means. 
Instead of using $k$-means, some works \cite{SR,ISR}
aim to compute the closest discrete solution regarding the Euclidean distance, 
\begin{equation} \label{problem_spectral_rotation}
    \min_{\bm G \in \mathcal{G}_{n \times c}, \bm R^T \bm R = \bm I} \mathcal J_{\textrm{ISR}} (\bm Y) =  \|\bm F_* \bm R - \bm G \|^2 . 
\end{equation}
Recently, some researchers \cite{DNC,MDNC} turn to directly solve the original non-convex 
problem, \textit{i.e.}, problem (\ref{problem_raw}), by means of the 
re-weighted optimization \cite{re-weighted}. 
However, as problem (\ref{problem_raw}) is non-convex and NP-hard, 
these methods are not guaranteed to approach the optimum.

\section{Methodology}
In this section, we formally show the deficiencies of the existing discretization methods at first, 
in order to clarify our motivation. 
Then we generally elaborate on the idea that aims to utilize the non-heuristic information 
to discretize the continuous solution. 
Finally, a specific case is used to testify the effectiveness of the idea. 
To begin with, we define two crucial variables: 
the discrete optimum 
$\bm G_* = \arg \min_{\bm G \in \mathcal{G}_{n \times c}} {\rm tr}(\bm G^T \bm L \bm G)$
and the closest discrete solution regarding Euclidean distance, 
$\bm G_\dag = \arg \min_{\bm G \in \mathcal{G}_{n \times c}, \bm R^T \bm R = \bm I} \|\bm G - \bm F_* \bm R\|^2 .$

\begin{figure}[t]
    \centering
    \includegraphics[width=0.9\linewidth]{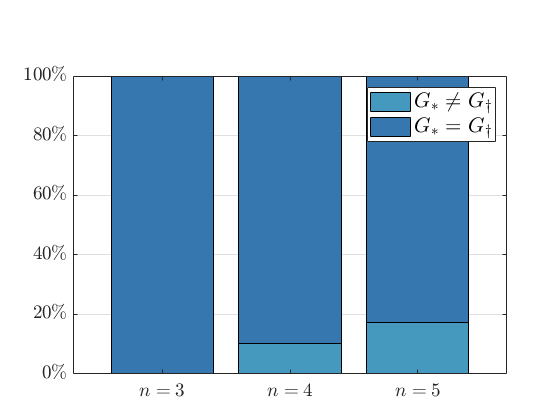}
    \caption{Empirical illustration of our motivation: $\bm G_\dag$ 
    represents the closest discrete solution regarding Euclidean distance and $\bm G_*$
    represents the optimal discrete solutions. We test on 10,000 randomly synthetic graphs 
    and the target number of clusters is set as 2. With the increase of $n$, 
    $\bm G_\dag \neq \bm G_*$ becomes more and more frequent, which indicates that 
    the conventional methods may be unreliable. }
    \label{figure_toy}
\end{figure}

\subsection{Our Motivation}
% In this section, we clarify the drawbacks of the existing two kinds of methods 
% to show the motivation more clearly. 
To quickly understand the limitation of spectral rotation, we firstly provide a specific instance as follows. 
Suppose that the graph matrix is 
\begin{equation}
    \bm S = 
    \left [
    \begin{array}{c c c c}
        0 & 0.5 & 0.1 & 0.8 \\
        0.5 &  0 & 0.4 & 0.2 \\
        0.1 & 0.4 &  0 & 0.5 \\
        0.8 & 0.2 & 0.5 &  0 \\
    \end{array}
    \right ], 
\end{equation}
and the target is to partition the graph into 2 clusters regarding 
the Ratio Cut problem. 
Since there are only 7 feasible discrete solutions, we 
can obtain the discrete optimum via the enumeration and the continuous 
optimum via the eigendecomposition, 
% They are given as below 
\begin{equation}
    \bm G_* = 
    \left [
    \begin{array}{c c}
        \frac{1}{\sqrt{2}} & 0 \\
        0 & \frac{1}{\sqrt{2}} \\
        0 & \frac{1}{\sqrt{2}} \\
        \frac{1}{\sqrt{2}} & 0 \\
    \end{array}
    \right ]
    ~ \textrm{and} ~
    \bm F_* \approx 
    \left [
    \begin{array}{c r}
        0.5 & 0.5556 \\
        0.5 & 0.0629 \\
        0.5 & -0.8073 \\
        0.5 & 0.1888
    \end{array}
    \right ]. 
\end{equation}
Nevertheless, the optimal solution of problem (\ref{problem_spectral_rotation}), 
\textit{i.e.}, 
% the closest discrete solution regarding Euclidean distance of $\bm F_*$, 
$\bm G_\dag$, 
is $[1/\sqrt{3}, 0; 1/\sqrt{3}, 0; 1/\sqrt{3}, 0; 0, 1]$. 
Moreover, we also simulate on random graphs composed of 3/4/5 vertices which 
is shown in Figure \ref{figure_toy}. 
With the increase of $n$, 
the proportion of $\bm G_\dag \neq \bm G_*$ becomes larger. 
Empirically, it may be an improper choice for the discretization of $\bm F_*$.  

On the other hand, it is groundless to simply employ the $k$-means. 
More formally, the $k$-means can be formulated as 
\begin{equation} \label{problem_kmeans}
    \begin{split}
        & \min _{\bm Y \in \mathcal{B}_{n \times c}, \bm M} \| \bm F_* - \bm Y \bm M \|^2 \\ 
        = & \min _{\bm Y \in \mathcal{B}_{n \times c}} \|\bm F_* - \bm Y(\bm Y^T \bm Y)^{-1} \bm Y^T \bm F_*\|^2 \\ 
        = & \min _{\bm Y \in \mathcal{B}_{n \times c}} \mathcal J_{k\textrm{-means}} (\bm Y) , 
        % \Leftrightarrow & \min _{\bm Y \in \mathcal{B}_{n \times c}} \|\bm F - \bm Y(\bm Y^T \bm Y)^{-1} \bm Y^T \bm F\|^2 , 
    \end{split}
\end{equation}
where 
$\bm M$ represents the centroid matrix and 
the transformation is based on taking the derivative \textit{w.r.t.} $\bm M$. 
In particular, for the Ratio Cut problem, the relationship between $k$-means and the improved spectral rotation 
can be stated more formally as follows. 
\begin{myLemma} \label{lemma}
    For any two matrices $\bm A, \bm B \in \mathbb R^{n \times c}$ that satisfy 
    $\bm A^T \bm A = \bm B^T \bm B = \bm I$, 
    the singular values of $\bm A^T \bm B$ are smaller than 1. 
\end{myLemma}
\begin{proof}
    For any $\bm u \in \mathbb{R}^c$, we have 
    \begin{equation}
        \frac{\|\bm A \bm u\|}{\|\bm u\|} = \sqrt \frac{\bm u^T \bm A^T \bm A \bm u}{\bm u^T \bm u} \leq 1.
    \end{equation}
    Accordingly, for any $\bm u \in \mathbb{R}^{c}$, 
    \begin{equation}
        \frac{\bm u^T \bm A^T \bm B \bm B^T \bm A \bm u}{\bm u^T \bm u} = \bm \alpha^T \bm B \bm B^T \bm \alpha \leq 1, 
    \end{equation}
    where $\bm \alpha = \bm A \bm u / \|\bm u\|$. Hence, the lemma is proved. 
\end{proof}
\begin{myTheorem} 
    For the spectral clustering using the Ratio Cut, 
    there exists a real number $0 \leq \epsilon < 3$ such that 
    \begin{equation}
        \mathcal J_{k\textrm{-means}} (\bm Y) \leq \mathcal{J}_{\textrm{ISR}} (\bm Y) \leq (1 + \epsilon) \mathcal J_{k\textrm{-means}} (\bm Y) .
    \end{equation}
    % where $\min \epsilon \leq 3$. 
\end{myTheorem}
\begin{proof}
    In the Ratio Cut, $G_{ij} = 1/\sqrt{|\mathcal{C}_j|}$ if the $i$-th 
    data point belongs to the $j$-th cluster. 
    Formally, $\bm G = f(\bm Y) = \bm Y (\bm Y^T \bm Y)^{-1/2}$. 
    Following the triangle inequality, 
    \begin{align*}
        \mathcal{J}_{\textrm{ISR}}^{\frac{1}{2}} (\bm Y) = & \min _{\bm R^T \bm R = \bm I} \|\bm F_*  -  \bm G \bm R^T\| \\
        \leq & \|\bm F_* - \bm Y(\bm Y^T \bm Y)^{-1} \bm Y^T \bm F_*\| \\
        & + \min _{\bm R^T \bm R = \bm I} \| \bm Y(\bm Y^T \bm Y)^{-1} \bm Y^T \bm F_* - \bm G \bm R^T\| \\
        = & \mathcal J_{k\textrm{-means}}^{\frac{1}{2}} (\bm Y) + \min _{\bm R^T \bm R = \bm I} \| \bm Y(\bm Y^T \bm Y)^{-1} \bm Y^T \bm F_* - \bm G \bm R^T\|.
        % = & \mathcal J_{k\textrm{-means}}^{\frac{1}{2}} (\bm Y) + \mathcal J_{t}J_{t}^{\frac{1}{2}} (\bm Y) . 
    \end{align*}
    % \begin{equation} \notag
    %     \begin{split}
    %     & \mathcal{J}_{k\textrm{-means}} 
    %     = \min _{\bm Y \in \mathcal{B}_{n \times c}} \|\bm F_* - \bm Y(\bm Y^T \bm Y)^{-1} \bm Y^T \bm F_*\|^2 \\
    %     \leq & \min_{\bm Y \in \mathcal{B}_{n \times c}} \|\bm F_* - \bm G \bm R\|^2 + \|\bm G \bm R - \bm Y(\bm Y^T \bm Y)^{-1} \bm Y^T \bm F_*\|^2 \\
    %     \leq & \mathcal{J}_{\textrm{ISR}} + \min _{\bm Y \in \mathcal{B}_{n \times c}, \bm R^T \bm R = \bm I} \|\bm G \bm R - \bm Y(\bm Y^T \bm Y)^{-1} \bm Y^T \bm F_*\|^2 . 
    %     \end{split}
    % \end{equation}
    Clearly, we can substitute $\bm G = \bm Y (\bm Y^T \bm Y)^{-1/2}$ into the latter term and get 
    \begin{align*}
        & \min _{\bm R^T \bm R = \bm I} \|\bm Y (\bm Y^T \bm Y)^{-\frac{1}{2}} \bm R^T - \bm Y(\bm Y^T \bm Y)^{-1} \bm Y^T \bm F_* \|^2 \\ 
        % = & \| \bm R - (\bm Y^T \bm Y)^{-\frac{1}{2}} \bm Y^T \bm F_*\|^2 \\
        = & {\rm tr}(\bm R \bm R^T) + {\rm tr}(\bm F_*^T \bm Y (\bm Y^T \bm Y)^{-1} \bm Y^T \bm F_*) \\ 
        & - \max _{\bm R \bm R = \bm I} 2 {\rm tr}(\bm R (\bm Y^T \bm Y)^{-\frac{1}{2}} \bm Y^T \bm F_*) . 
    \end{align*}
    Let $\bm M = (\bm Y^T \bm Y)^{-\frac{1}{2}} \bm Y^T \bm F_*$ 
    and $[\bm U, \bm \Sigma, \bm V] \leftarrow {\rm svd}(\bm M)$. 
    Note that 
    \begin{equation}
        \begin{split}
            {\rm tr}(\bm R \bm M) & = {\rm tr}(\bm R \bm U \bm \Sigma \bm V^T) = {\rm tr}(\bm \Sigma \bm V^T \bm R \bm U) \\ 
            & \leq \sum \Sigma_{ii} (\bm V^T \bm R \bm U)_{ii} \leq {\rm tr}(\bm \Sigma) , 
        \end{split}
    \end{equation}
    where $\bm V^T \bm R \bm U$ is orthonormal so that $(\bm V^T \bm R \bm U)_{ij} \leq 1$. 
    Accordingly, 
    \begin{equation}
        % & \min _{\bm R^T \bm R = \bm I} \|\bm G \bm R - \bm Y(\bm Y^T \bm Y)^{-1} \bm Y^T \bm F_*\|^2 \\ 
        \mathcal{J}_t (\bm Y) = c + {\rm tr}(\bm \Sigma^2) - 2 {\rm  tr}(\bm \Sigma) = \sum_{i} (1-\sigma_i)^2 , 
    \end{equation}
    where $\sigma_i$ is the $i$-th singular value. 
    On the other hand, 
    \begin{align*}
        \mathcal{J}_{k\textrm{-means}} (\bm Y) & = \|\bm F_* - \bm Y(\bm Y^T \bm Y)^{-1} \bm Y^T \bm F_*\|^2 \\
        & = c + {\rm tr}(\bm M^T \bm M) - 2 {\rm tr}(\bm M^T \bm M) \\
        & = c - {\rm tr}(\bm \Sigma^2) = \sum_i (1 - \sigma_i) (1 + \sigma_i) . 
    \end{align*}
    Accordingly, we have 
    \begin{equation} \notag
        \varepsilon \mathcal{J}_{k\textrm{-means}} (\bm Y) - \mathcal{J}_t (\bm Y) = \sum_i (1-\sigma_i) (\varepsilon - 1 + (\varepsilon+1)\sigma_i). 
    \end{equation}
    According to Lemma \ref{lemma}, $\sigma_i \leq 1$. 
    Therefore, if $\forall i, \sigma_i \geq \frac{1 - \varepsilon}{1 + \varepsilon}$ 
    (\textit{i.e.}, $\varepsilon \geq \max_i \frac{1-\sigma_i}{1+\sigma_i}$), 
    then $\mathcal{J}_t(\bm Y) \leq \varepsilon \mathcal J_{k\textrm{-means}}(\bm Y)$ . 
    Furthermore, we have $\mathcal{J}_{\textrm{ISR}} (\bm Y) \leq (1+\sqrt{\varepsilon})^2 \mathcal J_{k\textrm{-means}} (\bm Y)$. 
    Let $\epsilon = \varepsilon + 2 \sqrt{\varepsilon}$ 
    and the right inequality is proved. 

    On the other hand, another part can be easily proved by 
    \begin{equation}
        \begin{split}
            & \mathcal{J}_{\textrm{ISR}} (\bm Y) - \mathcal{J}_{k\textrm{-means}} (\bm Y) \\
            = & 2 c - 2 {\rm tr}(\bm \Sigma) - c + {\rm tr}(\bm \Sigma^2) 
            = \sum_i (1-\sigma_i )^2 \geq 0 , 
        \end{split}
    \end{equation}
    which completes the proof. 
\end{proof}
Based on the above theorem, we can conclude that the solutions that 
cause small $k$-means losses will also be relatively preferable solutions 
of spectral rotation. 
The conclusion can be formulated as 
\begin{myCorollary}
    For the Ratio Cut problem, if 
    $(1+\epsilon) \mathcal{J}_{k\textrm{-means}} (\bm Y_1) \leq \mathcal J_{k\textrm{-means}} (\bm Y_2)$ 
    (where $\epsilon$ is taken from Lemma \ref{lemma}) holds for any two clustering assignments $\bm Y_1$ and $\bm Y_2$, 
    then $\mathcal{J}_{\textrm{ISR}} (\bm Y_1) \leq \mathcal{J}_{\textrm{ISR}} (\bm Y_2)$. 
\end{myCorollary}

\begin{figure}
    \centering
    \includegraphics[width=0.9\linewidth]{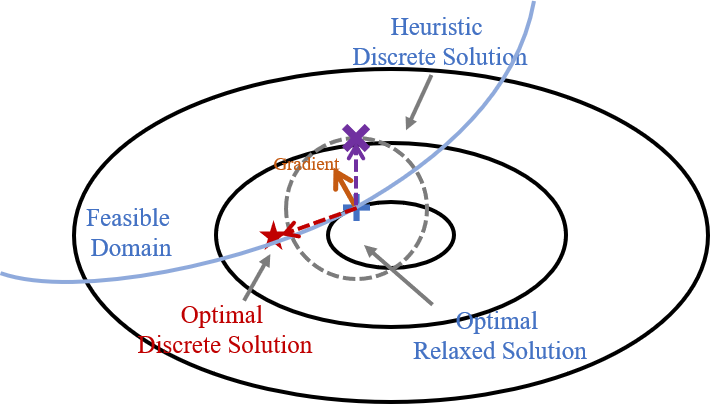}
    \caption{Visual demonstration of our motivation: All heuristic methods are 
    based on Euclidean distance, which is inconsistent with the metric used in the 
    original graph cut problem. The introduction of the gradient term can provide 
    information of the correct direction for optimization. }
    \label{figure_motivation}
\end{figure}

\subsection{A Non-Heuristic Method via Introducing Gradient}
Rethink the problem discussed above and we can find that the shortcoming of 
the existing methods are mainly caused by \textit{the neglect of the original optimization 
problem}. In other words, both problem (\ref{problem_spectral_rotation}) and 
(\ref{problem_kmeans}) is independent of the original objective defined 
in (\ref{problem_raw}). 
Therefore, the first problem we intend to address is \textbf{\textit{how to make the discretization algorithm aware 
of the original problem via introducing the gradient}}.

% As the unconstrained continuous problem is quadratic, Consider 
In the beginning, the residual is defined as 
$\bm \Delta_t = \bm G_t - \bm F_* \bm R_t$ where 
$\bm R_t = \arg \min_{\bm R^T \bm R = \bm I} \| \bm G_t - \bm F_* \bm R\|$. 
It should be pointed out that $\bm R_t$ is used since for any continuous solution $\bm F$, 
there exists a set $\{\bm F \bm R | \bm R^T \bm R = \bm I\}$ that leads to the same loss. 
With the definition of $\bm \Delta$, if $\|\bm \Delta_1\| = \|\bm \Delta_2\|$ 
and $\|\bm \Delta_1\|$ is small enough, the following inequality, 
\begin{equation}
    \langle \bm \Delta_1, \nabla_{\bm F = \bm F_* \bm R_1} \mathcal{L} \rangle
    < \langle \bm \Delta_2, \nabla_{\bm F = \bm F_* \bm R_2} \mathcal{L} \rangle , 
\end{equation}
will indicate that ${\rm tr}(\bm G_1^T \bm L \bm G_1) < {\rm tr}(\bm G_2^T \bm L \bm G_2)$. 
Figure \ref{figure_motivation} visually shows the core idea of our framework. 

The second important question is raised according to \cite{DNC}: \textbf{\textit{Is the continuous optimum $\bm F_*$ necessary?}} 
Or formally, is the term $\|\bm F_* \bm R - \bm G\|$ beneficial to the discretization algorithms?
To answer this question, we define a lower-bound metric, $\rho(\bm \Delta)$, of $\|\bm \Delta\|$ 
% \begin{equation}
%     \rho(\bm \Delta) = \|\|
% \end{equation}
due to that the Laplacian matrix $\bm L$ is positive semi-definite. 
Let $\bm \Delta_\dag$ and $\bm \Delta_*$ be the residual matrix with 
$\bm G = \bm G_\dag$ and $\bm G = \bm G_*$, respectively.
The following theorem shows that under this metric, 
$\rho(\bm \Delta_*)$ will not exceed $\rho(\bm \Delta_\dag)$ too much. 
\begin{myTheorem} \label{theo_delta}
    Given a discrete solution $\bm G$, define 
    $\rho(\bm \Delta) = \|\bm F_\bot^T \bm \Delta\|$ where the columns of $\bm F_\bot$ are 
    eigenvectors corresponding to non-zero eigenvalues of $\bm L$. 
    Then the following inequality holds: 
    % $\delta (\bm G) = \min _{\bm R} \|\bm F_* \bm R - \bm G\|^2$.
    % Suppose that $\bm G_*$ is the optimal solution and $\delta_* = \delta (\bm G_*)$. Then 
    \begin{equation}
        \rho^2(\bm \Delta_*) \leq \frac{\lambda_{\max}}{\lambda_{\min}} \rho^2 (\bm \Delta_\dag) + \mathcal{O}(1) , 
    \end{equation}
    where $\lambda_{\min}$ and $\lambda_{\max}$ represent the minimum and maximum non-zeros 
    eigenvalues, respectively. 
    % $C$ is a constant irrelevant to $\bm G_\dag$ and $\bm G_*$. 
\end{myTheorem}
\begin{proof}
    At first, we can decompose $\mathcal{L}(\bm G)$ as 
    \begin{align*}
        \mathcal{L}(\bm G) 
        % & = \mathcal{L}(\bm F_*) + \mathcal{L}(\bm \Delta) + 2 {\rm tr}(\bm \Delta^T \bm L \bm F_* \bm R) \\
        % & = - \mathcal{L}(\bm F_*) + \mathcal{L}(\bm \Delta) + 2 {\rm tr}(\bm G^T \bm L \bm F_* \bm R) \\
        & = - \mathcal{L}(\bm F_*) + \mathcal{L}(\bm \Delta) + 2 {\rm tr}(\bm G^T \bm F_* \bm \Lambda_c \bm R) ,
    \end{align*}
    where $\bm \Lambda_c \in \mathbb R^{c \times c}$ is a diagonal matrix with $c$ smallest eigenvalues of $\bm L$.
    Let $\bm F_0$ be the matrix consisting of eigenvectors corresponding to 0 
    and $\bm U = [\bm F_0, \bm F_\bot]$ denotes all eigenvectors of $\bm L$. 
    Therefore, $\mathcal{L}(\bm \Delta)$ can be rewritten as 
    \begin{align*}
        \mathcal{L}(\bm \Delta) = \|\bm \Lambda^\frac{1}{2} \bm U^T \bm \Delta_\dag\|^2 = \|0^\frac{1}{2} \bm F_0^T \bm \Delta\|^2 + \|\bm \Lambda_\bot^\frac{1}{2} \bm F_\bot^T \bm \Delta\|^2 
    \end{align*}
    where $\bm \Lambda_\bot$ is a diagonal matrix composed of all non-zero eigenvalues. 
    From the above formulation, we can derive the following inequality, 
    \begin{equation}
        \lambda_{\min} \rho^2(\bm \Delta) \leq \mathcal{L}(\bm \Delta) \leq \lambda_{\max} \rho^2(\bm \Delta) . 
    \end{equation}
    According to $\bm R \bm G^T \bm F_* = \bm V \bm \Sigma \bm V^T$ where $\bm \Sigma$ and $\bm V$ are the singular 
    value matrix and right-singular vector matrix, we have 
    \begin{align*}
        & {\rm tr}(\bm G^T_\dag \bm F_* \bm \Lambda_c \bm R_\dag) - {\rm tr}(\bm G^T \bm F_* \bm \Lambda_c \bm R) \\
        = & {\rm tr}(\bm R_\dag \bm G^T_\dag \bm F_* \bm \Lambda_c) - {\rm tr}(\bm R \bm G^T \bm F_* \bm \Lambda_c) \\
        = & {\rm tr}(\bm V_\dag \bm \Sigma_\dag \bm V_\dag^T \bm \Lambda_c) - {\rm tr}(\bm V \bm \Sigma \bm V^T \Lambda_c) \leq \sum_{i=1}^c \lambda_i . 
    \end{align*}
    Clearly, provided that 
    \begin{equation}
        \lambda_{\min} \rho^2(\bm \Delta) - \lambda_{\max} \rho^2 (\bm \Delta_\dag) \geq 2 \sum_{i=1}^c \lambda_i , 
    \end{equation}
    add the two inequality and we can obtain 
    \begin{equation}
        \mathcal{L}(\bm \Delta) - \mathcal{L}(\bm \Delta_\dag) \geq 2 {\rm tr}(\bm G^T_\dag \bm F_* \bm \Lambda_c \bm R_\dag) - 2 {\rm tr}(\bm G^T \bm F_* \bm \Lambda_c \bm R) , 
    \end{equation}
    which indicates $\mathcal{L}(\bm G) \geq \mathcal{L}(\bm G_\dag)$. 
    Since $\mathcal{L}(\bm G_*) \leq \mathcal{L}(\bm G_\dag)$, 
    we have 
    \begin{equation}
        \rho^2(\bm \Delta_*) \leq \frac{\lambda_{\max}}{\lambda_{\min}} \rho^2(\bm \Delta_\dag) + \mathcal{O}(1) . 
    \end{equation}
    % where $\textrm{Constant} = \sum_{i=1}^c \lambda_i / \lambda_{\min}$
    To sum up, the theorem is proved. 
\end{proof}
From the definition, $\rho(\bm \Delta)$ is a lower-bound of $\|\bm \Delta\|$ due to $\|\bm \Delta\|^2 = \|[\bm F_0, \bm F_\bot]^T \bm \Delta\|^2 \geq \|\bm F_\bot^T \bm \Delta\|^2 = \rho^2(\bm \Delta)$. 
It indicates that $\rho^2(\bm \Delta_*) \leq (\lambda_{\max} / \lambda_{\min}) \|\bm \Delta_\dag\|^2 + \mathcal{O}(1)$. 
Although it can not directly provide the connection between $\|\bm \Delta_*\|$ and $\|\bm \Delta_\dag\|$, 
the theoretical analysis implies that minimizing $\|\bm \Delta\|$ is beneficial to finding the optimum. 

According to the discussion about the two questions, we propose to optimize the following problem, 
\begin{equation} \label{problem_proposal}
    \min \limits_{\bm R^T \bm R = \bm I, \bm G \in \mathcal{G}_{n \times c}} \|\bm F_* \bm R - \bm G \|^2 
    + \eta \langle \bm \Delta, \nabla_{\bm F = \bm F_* \bm R} \mathcal{L} \rangle ,
\end{equation}
where $\eta$ is a trade-off parameter. 
The first term is based on Theorem \ref{theo_delta} and it aims to 
minimize $\|\bm \Delta\|$. 
The second term is the gradient term used to bridge the graph cut problems 
and discretization algorithms, which is much easier to optimize due to its 
first-order property. 
To sum up, 
compared with $k$-means and spectral rotation, 
the above method takes the original problem into account. 
On the other hand, it can speed up the optimization and is more possible to 
find a better discrete solution compared with the direct methods \cite{DNC,MDNC}. 
% Note that the first term is equivalent to the spectral rotation defined in 
% problem (\ref{problem_spectral_rotation}) due to the rotation invariance of Frobenius-norm. 

\subsection{A Specific Case for 2 Popular Graph Cut Functions}

In this subsection, we focus on a specific formulation, $\bm G = f(\bm Y) = \bm D^{\frac{1}{2}} \bm Y (\bm Y^T \bm D \bm Y)^{-\frac{1}{2}}$, 
which contains the most popular two graph cut problems, Ratio Cut and Normalized Cut. 
Clearly, if $\bm D = \bm I$, then the above formulation corresponds 
to Ratio Cut. And if $\bm D$ is diagonal and $\bm D_{ii}$ 
% = {\rm degree}(\bm x_i), 
is the degree of the $i$-th node, 
then it corresponds to Normalized Cut. 
Let $\mathcal{G}_{a \times b}' = \{\bm D^{\frac{1}{2}} \bm Y (\bm Y^T \bm D \bm Y)^{-\frac{1}{2}} | \bm Y \in \mathcal{B}_{a \times b}\}$. 
Take the derivative of problem (\ref{problem_raw_continuous}) and obtain 
$\nabla_{\bm F} \mathcal{L} = 2 \bm L \bm F$. 
% Apply the eigendecomposition to $\bm L$ and denote 
% $\bm \Lambda_*$ and $\bm F_*$ as the $c$ smallest eigenvalues and corresponding 
% eigenvectors, respectively. 
Expand Eq. (\ref{problem_proposal}) and we have 
\begin{align*}
    & \|\bm F_* \bm R - \bm D^{\frac{1}{2}} \bm Y (\bm Y^T \bm D \bm Y)^{-\frac{1}{2}}\|^2 
    + \eta \langle \bm \Delta,  2\bm L \bm F_* \bm R \rangle \\
    = & 2 {\rm tr} (\bm I)  - 2 {\rm tr}(\bm F_*^T \bm D^{\frac{1}{2}} \bm Y (\bm Y^T \bm D \bm Y)^{-\frac{1}{2}} \bm R^T) \\
    & - 2 \eta {\rm tr}(\bm F_*^T \bm L \bm F_*) 
    + 2 \eta {\rm tr}(\bm R^T \bm F_*^T \bm L \bm D^{\frac{1}{2}} \bm Y (\bm Y^T \bm D \bm Y)^{-\frac{1}{2}}) \\
    = & 2\sum_{i=1}^c (1 - \eta \lambda_i(\bm L)) \\
    & - 2 {\rm tr}((\bm R^T \bm F_*^T - \eta \bm R^T \bm F_*^T \bm L) \bm D^{\frac{1}{2}} \bm Y (\bm Y^T \bm D \bm Y)^{-\frac{1}{2}}) . 
\end{align*}
Therefore, problem (\ref{problem_proposal}) is equivalent to 
% \begin{equation} \label{problem_proposal_final}
%     \max _{\bm R^T \bm R = \bm I, \bm Y \in \mathcal{B}_{n \times c}} {\rm tr}((\bm R \bm F_*^T - \eta \bm R^T \bm F_*^T \bm L) \bm D^{\frac{1}{2}} \bm Y (\bm Y^T \bm D \bm Y)^{-\frac{1}{2}}). 
% \end{equation}
\begin{equation} \label{problem_proposal_final}
    \max _{\bm R^T \bm R = \bm I, \bm G \in \mathcal{G}_{n \times c}'} {\rm tr}((\bm R^T \bm F_*^T - \eta \bm R^T \bm F_*^T \bm L) \bm G) , 
\end{equation}
where we use the notation $\bm G$ for simplicity. 
In the subsequent part, we will elaborate on the optimization of problem (\ref{problem_proposal_final}), 
which is based on the alternative method. 
The optimum of subproblem to solve $\bm R$ is given by SVD. 
The subproblem regarding $\bm Y$ (\textit{i.e.}, $\bm G$) is solved by computing each optimal row 
vector greedily. 

\begin{algorithm}[t]
    \centering
    \caption{Algorithm to optimize problem (\ref{problem_proposal_final}). }
    \label{alg}
    \begin{algorithmic}
        \REQUIRE Continuous optimum $\bm F_*$, some Laplacian matrix $\bm L$, and balance coefficient $\eta$. 
        \STATE Randomly initialize $\bm Y$.
        \WHILE {not converge}
            \STATE Solve $\bm R$ by Eq. (\ref{solution_R}). 
            \FOR {$i=1,2,\ldots,n$}{
                \STATE Compute the loss gain $\delta_{ij}$ defined in Eq. (\ref{definition_loss_gain}). 
                \STATE Compute $\bm y^i_{(\textrm{new})}$ according to Eq. (\ref{solution_Y}). 
                \STATE Update the $i$-th row if updating by $\bm y^i_{(\textrm{new})}$ does not cause the trivial $\bm Y$. 
            }
            \ENDFOR
        \ENDWHILE
        \ENSURE Clustering assignments $\bm Y$.
    \end{algorithmic}   
\end{algorithm}

\begin{table*}[t]
    \centering
    \renewcommand\arraystretch{1.1}
    \caption{Objective values of different discretization methods with two 
    graph cut functions on 5 tiny datasets. As DNC directly solves the Normalized 
    Cut problem, it is not reported in columns of Ratio Cut. 
    OPT denotes the value caused by the optimal \textbf{discrete} solution computed by enumeration.}
    \label{table_values_synthetic}
    \begin{tabular}{l c c c c c c | c c c c c c c}
        \hline 

        \hline
        \multirow{2}{*}{Dataset} & \multicolumn{6}{c|}{Ratio Cut} & \multicolumn{7}{c}{Normalized Cut}  \\
         & OPT & KM & KM-norm & SR & ISR & Ours & OPT & KM & KM-norm & SR & ISR & DNC & Ours \\
        \hline  
        \hline
        $n=11$ & 0.5275 & \textbf{0.5275} & 0.5814 & 0.5814 & 0.5814 & 0.5814 & 0.5927 & 0.5927 & 0.5927 & 0.5927 & 0.5927 & 0.5927 & \textbf{0.5927} \\ 
        $n=12$ & 0.4846 & 0.4973 & 0.4973 & 0.4973 & 0.4973 & \textbf{0.4973} & 0.5225 & 0.5681 & 0.5681 & 0.5681 & 0.5681 & 0.5279 & \textbf{0.5279} \\
        $n=13$ & 0.4211 & 0.5034 & 0.5034 & 0.5034 & 0.4801 & \textbf{0.4626} & 0.4582 & 0.4741 & 0.4782 & 0.4741 & 0.4600 & 0.4648 & \textbf{0.4600} \\
        $n=14$ & 0.3734 & 0.3902 & 0.3902 & 0.3902 & 0.3902 & \textbf{0.3734} & 0.3733 & 0.3865 & 0.3865 & 0.3865 & 0.3865 & 0.3865 & \textbf{0.3733} \\ 
        $n=15$ & 0.3479 & 0.3679 & 0.3679 & 0.3679 & 0.3679 & \textbf{0.3679} & 0.3602 & 0.3663 & 0.3663 & 0.3663 & 0.3663 & 0.3602 & \textbf{0.3602} \\
        \hline 

        \hline
    \end{tabular}
\end{table*}

\begin{table*}[t]
    \centering
    \renewcommand\arraystretch{1.1}
    \setlength{\tabcolsep}{2mm}
    \caption{Objective values of different discretization methods on two 
    graph cut functions. As DNC directly solves the Normalized 
    Cut problem, it is not reported in columns of Ratio Cut. 
    OPT$_r$ denotes the value caused by the optimal solution of the \textbf{relaxed} problem.}
    \label{table_values}
    \begin{tabular}{l c c c c c c | c c c c c c c}
        \hline 

        \hline
        \multirow{2}{*}{Dataset} & \multicolumn{6}{c|}{Ratio Cut} & \multicolumn{7}{c}{Normalized Cut}  \\
         & OPT$_r$ & KM & KM-norm & SR & ISR & Ours  & OPT$_r$ & KM & KM-norm & SR & ISR & DNC & Ours \\
        \hline
        \hline
        JAFFE & 0.0709 & 0.0971 & 0.0971 & 0.0971 & 0.0971 & \textbf{0.0970} & 0.0708 & 0.0972 & 0.0972 & 0.0972 & 0.0972 & 0.7213 & \textbf{0.0972} \\ 
        UMIST & 0.1954 & 0.7406 & 0.6514 & 0.7169 & 0.7130 & \textbf{0.6873} & 0.1949 & 0.7083 & 0.6285 & 0.7503 & 0.7342 & 1.6826 & \textbf{0.6602} \\
        ORL   & 2.3285 & 3.2941 & 3.9446 & 3.9344 & 3.2849 & \textbf{3.2447} & 2.4201 & 3.3675 & 3.7821 & 3.6690 & 3.3737 & 6.7980 & \textbf{3.3445}  \\
        YALE & 1.2146 & 1.7805 & 2.2890 & 2.2120 & 1.9750 & \textbf{1.7723} & 1.2949 & 1.9592 & 2.2273 & 2.1963 & 1.9947 & 3.4948 & \textbf{1.9297} \\ 
        COIL20  & 0.0489 & 0.2402 & 0.2591 & 0.2644 & 0.2402 & \textbf{0.2364} & 0.0487 & 0.2406 & 0.2595 & 0.2623 & 0.2406 & 2.3078 & \textbf{0.2368} \\
        MSRA & 0.0021 & 0.0030 & 0.0043 & 0.0064 & 0.0030 & \textbf{0.0030} & 0.0021 & 0.0030 & 0.0037 & 0.0064 & 0.0030 & 1.4622 & \textbf{0.0030} \\        
        WINE & 0.0033 & 0.0154 & 0.0402 & 0.0419 & 0.0154 & \textbf{0.0154} & 0.0033 & 0.0154 & 0.0402 & 0.0337 & 0.0154 & 0.1011 & \textbf{0.0154} \\
        GLASS & 0.1256 & 0.2876 & 0.2785 & 0.2785 & 0.2876 & \textbf{0.2754} & 0.1250 & 0.2638 & 0.2793 & 0.2903 & 0.2627 & 0.4701 & \textbf{0.2627}  \\
        SEGMENT & 0.0025 & 0.0042 & 0.0176 & 0.0144 & 0.0042 & \textbf{0.0042} & 0.0025 & 0.0041 & 0.0176 & 0.0144 & 0.0044 & 0.7842 & \textbf{0.0041} \\
        USPS & 0.0538 & 0.1816 & 0.1911 & 0.1903 & 0.1819 & \textbf{0.1816} & 0.0536 & 0.1819 & 0.1903 & 0.1907 & 0.1825 & 2.1704 & \textbf{0.1817} \\
        Fashion & 0.0440 & 0.2029 & 0.2097 & 0.2158 & 0.2027 & \textbf{0.2025} & 0.0439 & 0.2099 & 0.2102 & 0.2162 & 0.2002 & 2.1193 & \textbf{0.2004} \\
        \hline 

        \hline
    \end{tabular}
\end{table*}

\subsubsection{Optimization}
When $\bm Y$ is fixed, the subproblem to optimize is 
\begin{equation}
    \max _{\bm R^T \bm R = \bm I} {\rm tr}(\bm R^T (\bm F_*^T \bm G - \eta \bm F_*^T \bm L \bm G)) .
\end{equation}
According to the orthogonal Procrustes theorem \cite{Procrustes}, the optimum of the above problem is 
\begin{equation} \label{solution_R}
    \bm R = \bm U \bm V^T, 
\end{equation}
where $[\bm U, \bm \Sigma, \bm V] \leftarrow \textrm{SVD}(\bm F_*^T \bm G - \eta \bm F_*^T \bm L \bm G)$. 
When $\bm R$ is fixed, then the subproblem to maximize is formulated as 
\begin{equation}
    % \max _{\bm Y \in \mathcal{B}_{n \times c}} 
    {\rm tr}(\bm M^T \bm D^{\frac{1}{2}} \bm Y (\bm Y^T \bm D \bm Y)^{-\frac{1}{2}}) 
    = \sum _{j=1}^c \frac{\sum_{i=1}^n \sqrt{D_{ii}} M_{ij} Y_{ij}}{\sqrt{\bm y_j^T \bm D \bm y_j}} ,
\end{equation}
where $\bm M = \bm F_* \bm R - \eta \bm L \bm F_* \bm R$. 
We optimize the above problem by greedily updating the $i$-th row, $\bm y^i$. 
% We consider the loss gain by setting $Y_{ij}$
Define the loss gain as 
\begin{equation} \label{definition_loss_gain}
    \begin{split}
    \delta_{ij} = & \frac{\sum_{k=1}^n \sqrt{D_{kk}} M_{kj} Y_{kj} + \sqrt{D_{ii}} M_{ij} (1-Y_{ij})}{\sqrt{\bm y_j^T \bm D \bm y_j + D_{ii} (1-Y_{ij})}} \\
    & - \frac{\sum_{k=1}^n \sqrt{D_{kk}} M_{kj} Y_{kj} - \sqrt{D_{ii}} M_{ij} Y_{ij}}{\sqrt{\bm y_j^T \bm D \bm y_j - D_{ii} Y_{ij}}} .
    \end{split}
\end{equation}
Note that the second term represents the objective value with $\bm y^i = \bm 0$
while the first term is the loss when $Y_{ij} = 1$. 
The $i$-th row is computed according to the maximum loss gain, 
\begin{equation} \label{solution_Y}
    \forall j, Y_{ij}^{\textrm{(new)}} = \mathbbm{1}\{j = \arg \max_j \delta_{ij}\}. 
\end{equation}

% \paragraph{Avoid Trivial Solutions} 
It should be emphasized that \textit{the above update formulation may violate the constraint 
of $\bm G \in \mathcal{G}_{n \times c}$}. 
Specifically speaking, the algorithm can partition samples into fewer clusters, 
\textit{i.e.}, $\bm g_i = 0$, so that it can obtain smaller function value. 
Therefore, a judgment is needed before the $i$-th row is updated, 
\begin{equation}
    \bm y^i = 
    \begin{cases}
        \bm y^i_{(\textrm{old})} & \textrm{if } \bm Y \textrm{ has zero columns} \\
        \bm y^i_{(\textrm{new})} & \textrm{else}
    \end{cases}
    , 
\end{equation}
where $\bm y^i_{(\textrm{new})}$ is defined in Eq. (\ref{solution_Y}) 
and $\bm y^i_{(\textrm{old})}$ represents the $i$-th row before this iteration.

The entire procedure is summarized in Algorithm \ref{alg}. 
Note that the judgment will not break the monotonous property of the optimization 
so that Algorithm \ref{alg} can converge into a local optimum. 
It should be also pointed out that we only update $\bm Y$ once 
in each iteration (rather than exactly solving $\bm Y$ in each iteration), 
which is inspired by the inexact ALM (IALM) \cite{IALM,LADMM}, to accelerate the algorithm.

\subsubsection{Complexity}
To solve the rotation matrix $\bm R$, it requires $\mathcal{O}(c^3)$ to 
perform SVD and the computation of $\bm L \bm F_*$ can 
be speeded up by the historical eigendecomposition used 
to compute $\bm F_*$. The total complexity of computing 
$\bm R$ is thus $\mathcal{O}(c^3 + n c^2)$. 
To update $\bm Y$ greedily, it needs $\mathcal{O}(nc^2)$ 
to update all rows of $\bm Y$. 
In sum, the computational complexity of Algorithm \ref{alg} is $\mathcal{O}(c^3 + n c^2)$.

\section{Experiments}
In this section, we aim to empirically investigate whether the proposed 
non-heuristic method works in practice. 
The primary criterion is whether the discrete solutions obtained 
by the proposed methods cause smaller losses compared with other 
existing methods. 

\begin{table}[t]
    \centering
    \renewcommand\arraystretch{1.0}
    \setlength{\tabcolsep}{4mm}
    \caption{Information of Datasets}
    \label{table_datasets}
    \begin{tabular}{l c c c}
        \hline 

        \hline
        Dataset & \# Samples & \# Features & \# Classes \\ 
        \hline
        \hline 
        JAFFE & 213 & 1,024 & 10 \\
        UMIST & 575 & 1,024 & 20 \\
        ORL & 400 & 1,024 & 40 \\
        YALE & 165 & 1,024 & 15 \\
        COIL20 & 1,440 & 1,024 & 20 \\
        MSRA & 1,799 & 256 & 12 \\
        WINE & 178 & 13 & 3 \\ 
        GLASS & 214 & 9 & 6 \\
        SEGMENT & 2,310 & 19 & 7 \\
        USPS & 9,298 & 256 & 10 \\
        Fashion & 10,000 & 784 & 10\\
        \hline 

        \hline
    \end{tabular}
\end{table}

\begin{table*}[t]
    \centering
    \renewcommand\arraystretch{1.1}
    \setlength{\tabcolsep}{3mm}
    \caption{Clustering accuracy of different discretization methods on two 
    graph cut functions. Similarly, DNC on Ratio Cut problem is neglected.}
    \label{table_acc}
    \begin{tabular}{l c c c c c | c c c c c c c c}
        \hline 

        \hline
        \multirow{2}{*}{Dataset} & \multicolumn{5}{c|}{Ratio Cut} & \multicolumn{6}{c}{Normalized Cut}  \\
        & KM & KM-norm & SR & ISR & Ours & KM & KM-norm & SR & ISR & DNC & Ours \\
        \hline
        \hline
        JAFFE & 0.9671 & 0.9671 & 0.9671 & 0.9671 & \textbf{0.9671} & 0.9671 & 0.9671 & 0.9671 & 0.9671 & 0.5399 & \textbf{0.9671} \\ 
        UMIST & 0.5670 & \textbf{0.6678}& 0.5826 & 0.5722 & 0.5965 & 0.6139 & 0.6278 & 0.6122 & 0.5930 & 0.3391 & \textbf{0.6296} \\
        ORL & 0.6150 & \textbf{0.6725} & 0.6475 & 0.6475 & 0.6475 & 0.6375 & \textbf{0.6850} & 0.6750 & 0.6625 & 0.4550 & 0.6625 \\
        YALE & 0.4303 & \textbf{0.4727} & 0.4606 & 0.4667 & 0.4606 & 0.4848 & \textbf{0.4909} & 0.4788 & 0.4606 & 0.2667 & 0.4848 \\ 
        COIL20 & 0.8396 & 0.8451 & \textbf{0.8451} & 0.8396 & 0.8403 & 0.8396 & 0.8451 & \textbf{0.8458} & 0.8396 & 0.1785 & 0.8396  \\
        MSRA & 0.5737 & 0.5614 & 0.5737 & 0.5737 & \textbf{0.5737} & 0.5737 & 0.5737 & 0.5737 & 0.5737 & 0.1684 & \textbf{0.5737} \\ 
        WINE & 0.7135 & 0.7303 & 0.6910 & 0.7135 & \textbf{0.7303} & 0.7135 & 0.7303 & 0.7247 & 0.7135 & 0.4775 & \textbf{0.7303} \\
        GLASS & 0.4019 & 0.4065 & 0.4065 & 0.4019 & \textbf{0.4299} & 0.4579 & 0.4720 & \textbf{0.4720} & 0.4533 & 0.3738 & 0.4579  \\
        SEGMENT & 0.4026 & \textbf{0.5104} & 0.4221 & 0.4364 & 0.4026 & 0.4026 & \textbf{0.5104} & 0.4021 & 0.4026 & 0.1779 & 0.4030 \\
        USPS & 0.6699 & \textbf{0.6707} & 0.6705 & 0.6698 & 0.6698 & 0.6692 & \textbf{0.6703} & 0.6703 & 0.6696 & 0.1738 & 0.6696 \\
        Fashion & 0.5395 & 0.5270 & 0.5356 & 0.5395 & \textbf{0.5404} & 0.5390 & 0.5275 & 0.5345 & 0.5387 & 0.1058 & \textbf{0.5404} \\
        \hline 

        \hline
    \end{tabular}
\end{table*}

\begin{table*}[t]
    \centering
    \renewcommand\arraystretch{1.1}
    \setlength{\tabcolsep}{3mm}
    \caption{Normalized mutual information of different discretization methods with two 
    graph cut functions.}
    \label{table_nmi}
    \begin{tabular}{l c c c c c | c c c c c c c c}
        \hline 

        \hline
        \multirow{2}{*}{Dataset} & \multicolumn{5}{c|}{Ratio Cut} & \multicolumn{6}{c}{Normalized Cut}  \\
        & KM & KM-norm & SR & ISR & Ours & KM & KM-norm & SR & ISR & DNC & Ours \\
        \hline
        \hline
        JAFFE & 0.9623 & 0.9623 & 0.9623 & 0.9623 & \textbf{0.9623} & 0.9623 & 0.9623 & 0.9623 & 0.9623 & 0.5063 & \textbf{0.9623} \\ 
        UMIST & 0.7799 & \textbf{0.8292} & 0.7918 & 0.7949 & 0.7918 & 0.8053 & 0.8109 & 0.8025 & 0.7934 & 0.4187 & \textbf{0.8140} \\
        ORL & 0.7922 & \textbf{0.8196} & 0.8108 & 0.8150 & 0.7987 & 0.8040 & \textbf{0.8212} & 0.8167 & 0.8098 & 0.6270 & 0.8207 \\
        YALE & 0.4968 & \textbf{0.5447} & 0.5255 & 0.4981 & 0.5023 & 0.5347 & 0.5305 & \textbf{0.5446} & 0.5108 & 0.2794 & 0.5209  \\ 
        COIL20 & 0.8981 & 0.8981 & 0.8981 & 0.8981 & \textbf{0.8990} & 0.8981 & 0.8981 & 0.8981 & 0.8981 & 0.1946 & \textbf{0.8981} \\
        MSRA & 0.7110 & 0.7110 & 0.7110 & 0.7110 & \textbf{0.7110} & 0.7110 & 0.7110 & 0.7110 & 0.7110 & 0.0965 & \textbf{0.7110}  \\ 
        WINE & 0.4195 & 0.4214 & 0.4016 & 0.4195 & \textbf{0.4214} & 0.4195 & 0.4214 & \textbf{0.4368} & 0.4195 & 0.1181 & 0.4124 \\
        GLASS & 0.2978 & \textbf{0.3344} & 0.2978 & 0.2978 & 0.3095 & 0.2959 & \textbf{0.3415} & 0.2942 & 0.2942 & 0.2515 & 0.2960 \\
        SEGMENT & 0.4312 & \textbf{0.5256} & 0.4249 & 0.4312 & 0.4312 & 0.4312 & 0.5256 & 0.4249 & 0.4312 & 0.0498 & 0.4325 \\
        USPS & 0.8276 & 0.8257 & 0.8259 & 0.8276 & \textbf{0.8276} & 0.8264 & 0.8264 & 0.8261 & 0.8276 & 0.0144 & \textbf{0.8276} \\
        Fashion & 0.5859 & 0.5935 & \textbf{0.5944} & 0.5856 & 0.5909 & 0.5854 & 0.5932 & \textbf{0.5938} & 0.5850 & 0.0061 & 0.5908 \\ 
        \hline 

        \hline
    \end{tabular}
\end{table*}

\subsection{Datasets}
Totally 11 real datasets are used to verify the effectiveness of the proposed 
methods, including JAFFE \cite{JAFFE}, UMIST, \cite{UMIST}, ORL \cite{ORL}, 
YALE \cite{YALE}, COIL20 \cite{COIL20}, MSRA \cite{MSRA}, three UCI datasets \cite{UCI} (WINE, GLASS, and SEGMENT), 
USPS \cite{USPS}, and test set of Fashion MNIST \cite{FashionMNIST} (denoted by Fashion). 
The details can be found in Table \ref{table_datasets}. 
To better show the performance of different discretization methods 
and guarantee reproducibility, we choose the first 10--15 samples 
to show the difference between the optimal discrete solution and solutions 
returned by diverse discretization methods.

\subsection{Experimental Settings}
To fairly show the feasibility of introducing gradient information, 
we collect 4 different methods as competitors, including $k$-means (\textit{KM}) \cite{NCut}, 
$k$-means on normalized $\bm F_*$ (\textit{KM-norm}) \cite{SC}, 
spectral rotation (\textit{SR}) \cite{SR}, improved spectral rotation (\textit{ISR}) \cite{ISR}, 
and directly solving normalized cut (\textit{DNC}) \cite{DNC}. 
Note that the investigations of discretization methods are limited 
though spectral clustering has been proposed for a long time 
and has been extensively studied. 
The graphs used in experiments are constructed according to the weight computation 
proposed in \cite{CAN} 
\begin{equation}
    S_{ij} = \frac{\max (d_{i\cdot}^{(k+1)} - d_{ij}, 0)}{ \sum _{j=1}^k d_{i\cdot}^{(k+1)} - d_{i \cdot }^{(j)}} , 
\end{equation} 
where $d_{ij} = \|\bm x_i - \bm x_j\|^2$ and $d_{i \cdot}^{(k)}$ 
represents the $k$-th smallest of $\{d_{ij}\}_{j=1}^n$. 
The sparsity, $k$, of neighbors is set as 10 in our experiments. 
We use the most popular two cut functions, Ratio Cut and Normalized Cut. 
The only hyper-parameter, $\eta$, of the proposed method is searched from 
$\{10^{-3}, 10^{-2}, 10^{-1}, 10^{0}, 10^{1}\}$. 
% Since too large $\eta$ may mislead the algorithm, 

\textbf{The primary metric is the value of graph cut functions} defined in problem (\ref{problem_raw}) 
since the goal of the discretization algorithms is to find solutions that minimize problem (\ref{problem_raw}). 
In addition, we also report the function values of continuous optimum, which 
should be regarded as the baseline. 
As the proposed method without the judgment may decrease the function value 
by reducing the number of clusters, we also report the clustering accuracy 
and normalized mutual information though 
\textit{the small function values do not always correspond to better clustering partitions}. 
Note that if the proposed method returns a partition with less than $c$ clusters, 
then the clustering accuracy can not be computed.

The code of DNC is downloaded from the homepage of the authors and the codes of 
other methods are implemented under MATLAB 2019b. 

\begin{figure*}[t]
    \centering
    \subcaptionbox{JAFFE}{
        \includegraphics[width=0.23\linewidth]{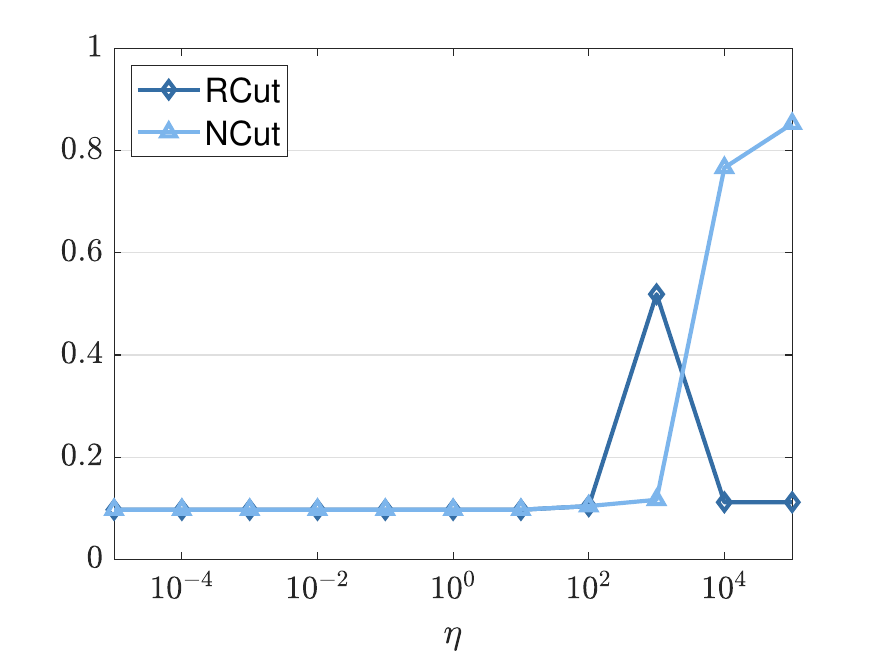}
    }
    \subcaptionbox{UMIST}{
        \includegraphics[width=0.23\linewidth]{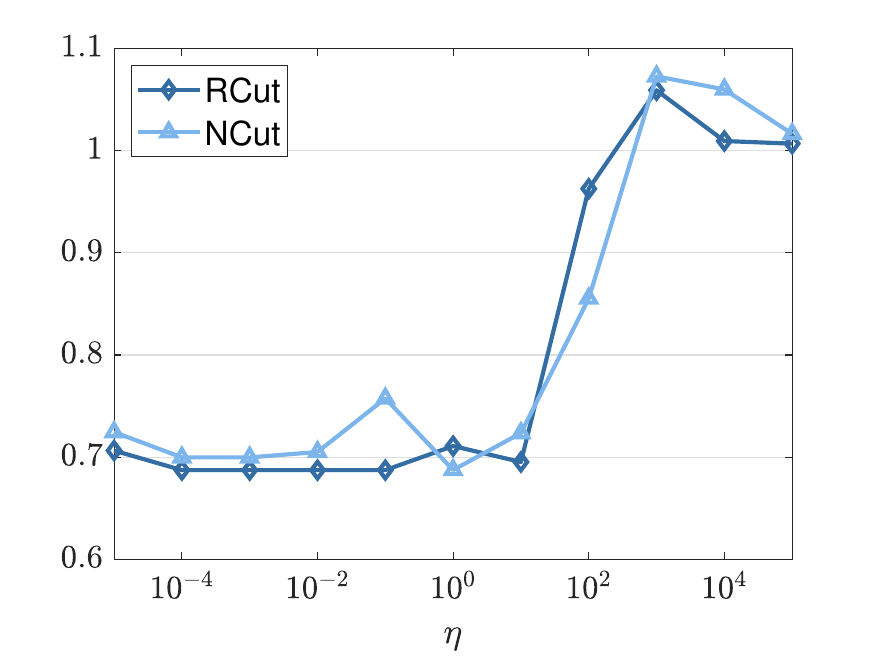}
    }
    \subcaptionbox{ORL}{
        \includegraphics[width=0.23\linewidth]{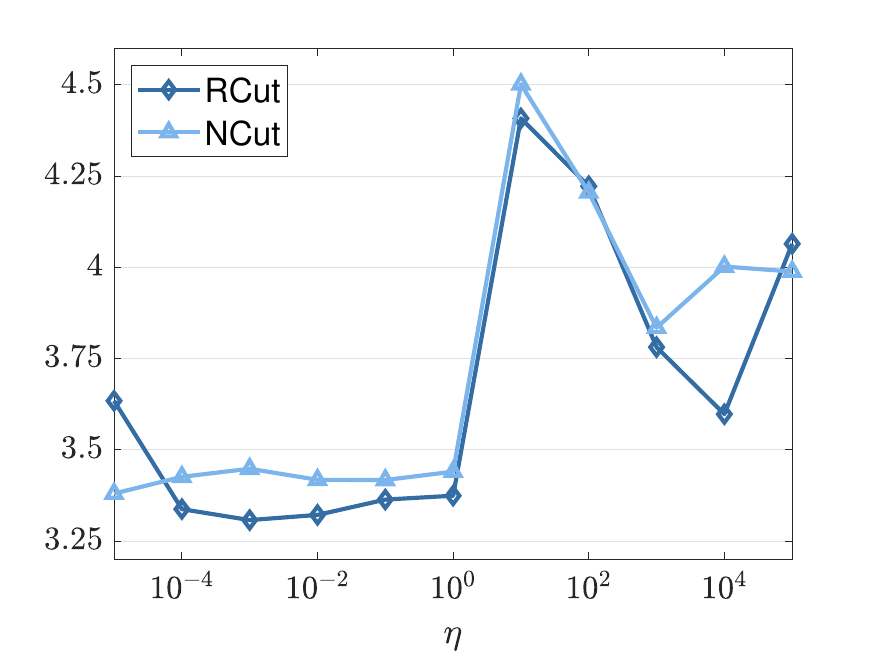}
    }
    \subcaptionbox{YALE}{
        \includegraphics[width=0.23\linewidth]{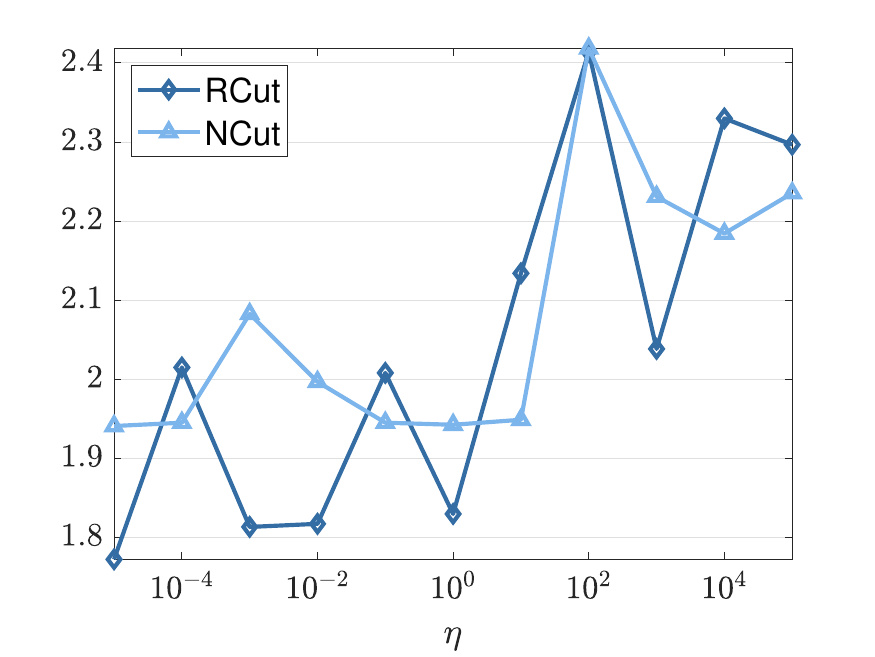}
    }
    \caption{The sensitivity study of function values to balance coefficient $\eta$. The $y$-axis represents the objective value. }
    \label{figure_sensitivity}
\end{figure*}

\begin{figure}
    \centering
    \subcaptionbox{JAFFE}{
        \includegraphics[width=0.45\linewidth]{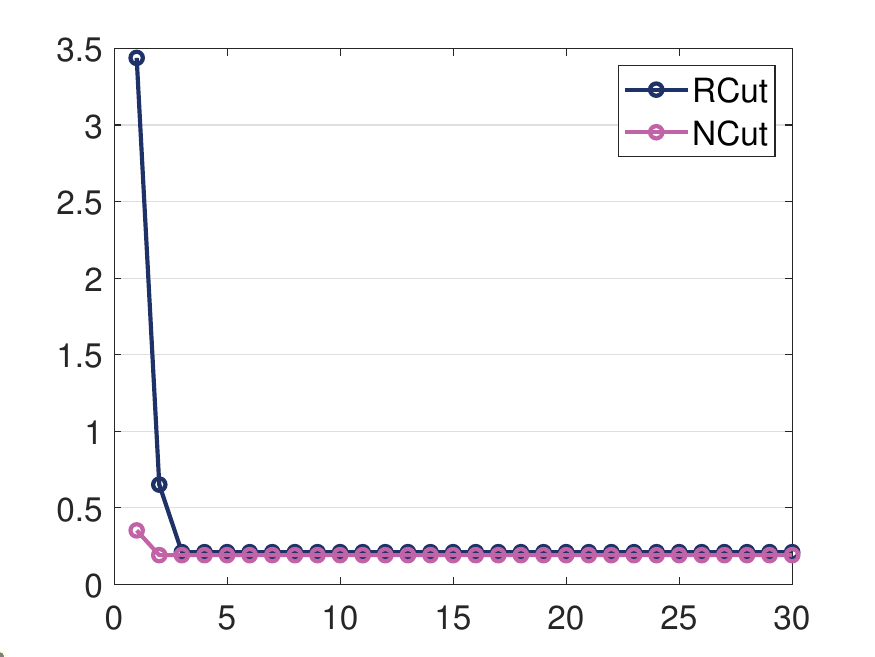}
    }
    \subcaptionbox{UMIST}{
        \includegraphics[width=0.45\linewidth]{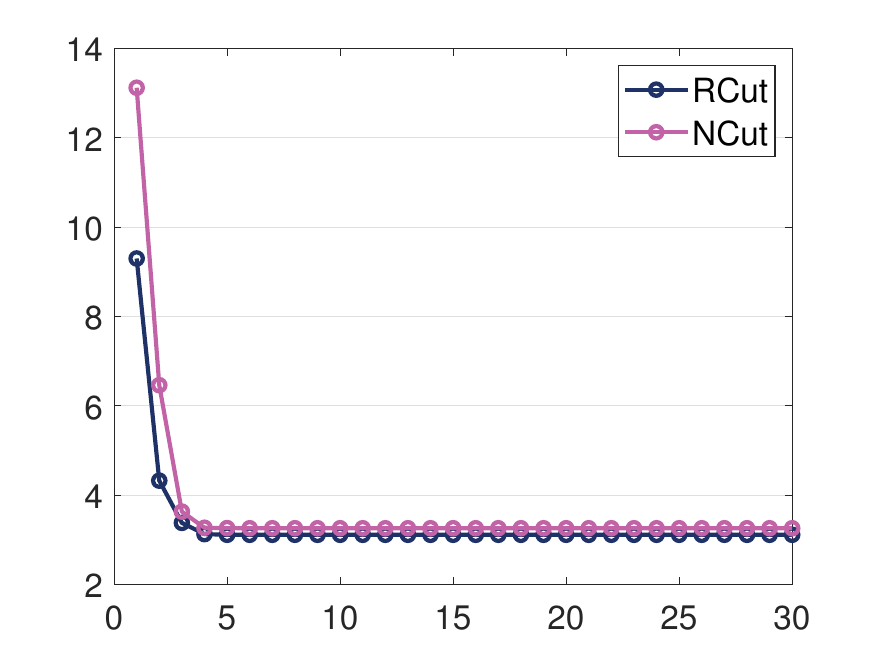}
    }
    \caption{Convergence curves of the proposed method. Although the algorithm is an inexact algorithm (similar to IALM), 
    the algorithm converges rapidly within 10 iterations. }
    \label{figure_convergence}
\end{figure}

\subsection{Main Results}
The objective values on tiny subsets of USPS are shown in Table \ref{table_values_synthetic}. 
We can find that the proposed method finds better discrete solutions when $n>11$. 
The values of Ratio Cut and Normalized Cut problems on 11 real datasets are summarized in Table 
\ref{table_values}. Clearly, our proposed method achieves the smallest values on 
all datasets. DNC starts from an arbitrary solution so that it usually 
can not find a satisfactory solution of an NP-hard problem. 
Although the intention of ISR is reasonable, it does not always outperform 
other competitors. On the one hand, the greedy optimization of ISR is not 
guaranteed to find the optimal solution. On the other hand, 
ISR is irrelevant to the original problem so that even the optimal solution obtained 
by ISR may be also undesired. 
The comparison between ISR and our method sufficiently shows the 
importance of the gradient term.

The clustering accuracy and normalized mutual information are reported 
in Tables \ref{table_acc} and \ref{table_nmi}. 
One may wonder why the proposed method sometimes achieves remarkable function values 
but does not obtains the best clustering metrics and concern the effectiveness of the proposed method. 
However, a fundamental assumption is that the used graph cut model is suitable 
for clustering on these datasets, which indicates that the better solution 
leads to better clustering results. 
If a discrete solution causes a small value but results in a bad clustering partition, 
it indicates the inappropriateness of the graph cut problem. 
In other words, it is not the mission of discretization algorithms to 
focus on how to improve the clustering metrics. 
As shown in Table \ref{table_values}, 
the proposed method always finds smaller discrete solutions compared with 
other existing methods. 
It means that the proposed method is a competent technique for discretization. 
It should be pointed out that we do not tune the construction of graphs for better clustering metrics 
since it is not the key to evaluating the proposed discretization method. 

In addition, we also show the convergence curves on JAFFE and UMIST in 
Figure \ref{figure_convergence} to show the impact of the inexact setting. 
Remark that compared with ISR, Algorithm \ref{alg} is inexact since the entire 
$\bm Y$ is only updated once in each iteration, 
which is similar to IALM \cite{IALM}. 
Nevertheless, the algorithm still converges fast, usually within 10 iterations.

\subsection{Sensitivity of Parameter}
To study the impact of the only hyper-parameter $\eta$, we testify 
the proposed method with different $\eta$ from a wider range, 
$\{10^{-5}, 10^{-4}, \ldots, 10^{5}\}$. 
The sensitivity curves are shown in Figure \ref{figure_sensitivity}. 
From the figure, we can conclude that too large $\eta$ causes instability, 
which is similar to the reason for the oscillation in neural networks \cite{Adam}. 
Apparently, $\eta = 10^{-3}$ is usually a desirable setting and we recommend to 
use this setting by default. 
It should be pointed out that $\eta$ can be also set by the simple search in practice. 
Although the supervised information is not provided in clustering, 
we can easily compute the value of different graph cut functions. 
Therefore, how to set an appropriate $\eta$ is not a problem in practice.

\section{Conclusion and Future Works} 
In this paper, we focus on how to design a non-heuristic discretization 
algorithm to outperform the existing methods. 
The idea is motivated by the fact that all existing methods split the 
original graph cut problems and the final discretization. 
We first theoretically and empirically show the drawbacks of existing discretization 
algorithms and therefore propose a first-order term to obtain 
the preferable discrete solution and meanwhile reduce the difficulty of 
solving the original NP-hard problem. 
We also theoretically point out the importance of the continuous optimum. 
Extensive experiments strongly support the theoretical analysis. 
The proposed method obtains significant improvements on all datasets and 
achieves state-of-the-art results. 

Although the impact of $\eta$ is empirically investigated, 
the theoretical analysis of $\eta$ is lacking, which is a focus problem 
in the conventional first-order gradient algorithms. 
In future work, a core topic is how to provide a theoretical range 
of $\eta$ and whether dynamically changing $\eta$ is feasible. 
% It results in the mismatch 
% The optimal solutions of all existing discretization algorithms frequently 
% do not match the optimum of graph cut problems. 

\bibliographystyle{IEEEtran}

\bibliography{citations.bib}

\end{document}